\documentclass{article}

\usepackage{microtype}
\usepackage{graphicx}
\usepackage{subfigure}
\usepackage{booktabs} 

\usepackage{amssymb}
\usepackage{enumitem}
\usepackage{amsmath}
\usepackage{amsthm}
\usepackage{natbib}
\newtheorem{thm}{Theorem}

\usepackage{hyperref}

\usepackage[accepted]{icml2021}

\icmltitlerunning{Comparison and Unification of Three Regularization Methods in Batch Reinforcement Learning}

\begin{document}

\twocolumn[
\icmltitle{Comparison and Unification of Three Regularization Methods in Batch Reinforcement Learning}



\icmlsetsymbol{equal}{*}

\begin{icmlauthorlist}
\icmlauthor{Sarah Rathnam}{am}
\icmlauthor{Susan A. Murphy}{cs,stat}
\icmlauthor{Finale Doshi-Velez}{cs}
\end{icmlauthorlist}

\icmlaffiliation{am}{Department of Applied Mathematics, Harvard University}
\icmlaffiliation{cs}{Department of Computer Science, Harvard University}
\icmlaffiliation{stat}{Department of Statistics, Harvard University}

\icmlcorrespondingauthor{Sarah Rathnam}{sarah\_rathnam@g.harvard.edu}

\icmlkeywords{Machine Learning, ICML}

\vskip 0.3in
]



\printAffiliationsAndNotice{} 

\begin{abstract}

In batch reinforcement learning, there can be poorly explored state-action pairs resulting in poorly learned, inaccurate models and poorly performing associated policies. Various regularization methods can mitigate the problem of learning overly-complex models in Markov decision processes (MDPs), however they operate in technically and intuitively distinct ways and lack a common form in which to compare them. This paper unifies three regularization methods in a common framework-- a weighted average transition matrix. Considering regularization methods in this common form illuminates how the MDP structure and the state-action pair distribution of the batch data set influence the relative performance of regularization methods. We confirm intuitions generated from the common framework by empirical evaluation across a range of MDPs and data collection policies. 
\end{abstract}

\section{Introduction}
\label{intro}
In certainty-equivalence reinforcement learning, the estimated model is treated as accurate when finding the optimal policy, without taking into account model uncertainty \cite{goodwin2014adaptive}. Consequently, when acting according to certainty-equivalence control, we risk finding a policy tailored to a model that is overly-expressive for the amount of data. 
This is especially problematic in a batch setting, as further exploration is not possible to improve the model. 

Many regularization methods address the problem of overfitting, for example reducing the planning horizon, using the posterior mean transition matrix under a Bayesian prior, or adding stochasticity to policies during planning. However, a challenge arises in understanding how they relate and choosing between them because regularization methods act on different elements of the MDP.  In the methods listed above, a reduced planning horizon modifies the discount factor, planning using the posterior mean transition matrix modifies the transition matrix, and planning over the set of stochastic policies modifies the set of policies over which we optimize.  Furthermore, their interpretations differ, for instance in the previously mentioned cases: decreasing the planning horizon, infusing outside information into the model, and planning over a stochastic set of policies.

Given certain constraints, the posterior mean transition matrix under a Bayesian prior is equivalent to a weighted average of the maximum likelihood estimator (MLE) transition matrix and the transition matrix implied by the prior.  Similarly, we express the other two regularization methods above as a weighted average between the MLE transition matrix and a regularization matrix of another form.  In this common Bayesian-like form, instead of comparing across disparate elements of the MDP, we can simply compare the form of the regularization matrix in each case and select the one that is most appropriate for the situation.  This framing suggests that a uniform Bayesian prior performs better in an MDP with densely-interconnected states, a lower discount factor performs better when balancing goals of different timescales, and planning over stochastic policies is preferable to avoid a catastrophic outcome. Simulations confirm that these hypotheses hold in many cases, but also underscore the need to take the data collection policy as well as the MDP into account when selecting a regularization method.

\section{Regularization in Certainty-Equivalence RL: Background and Related Work}

\paragraph{Bayesian Prior as Regularization}
A prior encodes expert knowledge, information from previous studies, or other outside information.  We can also view a prior on the transition function as a form of regularization since it forces the model not to overfit when data is limited \cite{poggio1990networks}.  In this paper, we consider planning using the posterior mean of the transition matrix under a Dirichlet prior as a regularized form of the transition matrix. 

\paragraph{Discount Regularization}
\citet{jiang} demonstrate that using a lower discount factor often leads to learning a policy that performs better than the one learned using the true discount factor. They prove that a lower discount factor restricts planning to a less complex set of policies, thereby avoiding overfitting.  They further demonstrate that the benefit of a lower discount factor is increasingly pronounced in cases where the model is estimated from a smaller data set. \citet{amit} refer to this concept as ``discount regularization,'' a term which we will use here. 

\paragraph{Planning over $\epsilon$-Greedy Policies}
\citet{arumugam} propose a regularization method where planning is conducted over the set of $\epsilon$-greedy policies rather than deterministic policies.  The added stochasticity prevents tailoring the policy too closely to the model. Like discount regularization, planning over $\epsilon$-greedy policies restricts the class of policies that can be optimal \cite{arumugam}.

\paragraph{Related Work: Other Regularization Methods}
Beyond the methods included in our unified framework, state aggregation maps the true MDP to a simpler, abstract representation. States are grouped by characteristics such as action-value function or optimal action \cite{state_abs}. Another method, $L_2$ regularization, introduces a complexity penalty, balancing a simpler model against one that fits the data more closely.  For example, \citet{amit} provide a framework to unify discount regularization with $L_2$ regularization in TD learning. Their use of $L_2$ regularization penalizes large value estimates, encouraging consistent value estimates across state-action pairs. In contrast, we frame methods as regularizing the transition matrix, thereby restricting model complexity.

\section{Notation and Definitions}
Methods in this paper are applied in a finite MDP setting. An MDP $M$ is characterized by $<S,A,R,T,\gamma>$, defined as follows. $S$: State space of size $N$. $A$: Action space. $R(s)$: Reward function. $R$ generally maps each state-action pair to a real-valued reward. In this paper, we consider rewards as a function of states only. $T(s'|s,a)$: Transition function, mapping each state-action pair to a probability distribution over successor states. $\gamma$: Discount factor, $0 \leq \gamma<1$. We assume $T$ and $R$ are unknown and estimated from the data. 

\section{Unification: Regularization as a Weighted Average Transition Matrix}
Each method above modifies a different element of the MDP.  To compare, we frame each as a weighted average of the MLE transition matrix and a matrix of another form. In this framework, we can compare by analyzing the matrix that is averaged with the MLE in each case. 

\paragraph{Dirichlet Prior}
We consider a Dirichlet distribution over the vector of successor state probabilities for a state-action pair, $T(s,a) = \langle p_1, ...,p_N \rangle$.  We assume prior $P_{prior}(T(s,a)) = \text{Dirichlet}(\langle \alpha_1,...,\alpha_N \rangle)$. The posterior mean can be expressed as a weighted average of $\hat{T}_{MLE}(s,a)$, the MLE of $T(s,a)$, and $T_{\text{prior mean}}(s,a)$, the transition matrix implied by the prior:
\begin{align*}\label{eq:1}
T_{\text{post mean}}(s,a)=(1-\epsilon) \hat{T}_{MLE}(s,a) + \epsilon T_{\text{prior mean}}(s,a)
\end{align*}
where $\epsilon=\frac{\sum \alpha_i}{\sum c_i + \sum \alpha_i}$ and $c_i$ is the transition count from state $s$ to state $i$ in the data set. 

The expression above is written for a single state $s$. To express the matrix $T_{\text{post mean}}(a)$ as a weighted average of the MLE and the prior transition matrix, $\epsilon$ must be equal for all states. 
If we assume (1) $\sum c_i$ equal across all states for given action $a$ (uniform visits), and (2) $\sum \alpha_i$ equal across all states for given action $a$ (identical priors), then we can write the matrix of posterior means as
\begin{equation}
\hat{T}(a)=(1-\epsilon) \hat{T}_{MLE}(a) + \epsilon T_{\text{prior mean}}(a)
\label{eq:dirichlet}
\end{equation}
(Full derivation in Appendix~\ref{appx:dirchlet}.)  Condition (2) holds for the choice of a uniform prior in empirical examples, however condition (1), uniform visits, is restrictive and unrealistic. We consequently do not enforce uniform visits in examples, however the weighted average form still provides insight in comparing this regularization form to others.

\paragraph{Discount Regularization}
To express discount regularization in the form of Equation \ref{eq:dirichlet}, consider the matrix form of the Bellman equation $V = R + \gamma T V$.  Let $\gamma_l < \gamma$ be the lower value of the discount factor used for regularization. We write $\gamma_l T$ from the Bellman equation under discount regularization as the product of $\gamma$, the true discount factor for the MDP, and a weighted average matrix:
\begin{align*}
\gamma_l T = \gamma[(1-\epsilon) T + \epsilon T_{zeros}]
\end{align*}
where $T_{zeros}$ is a matrix of zeros and $\epsilon=\frac{\gamma-\gamma_l}{\gamma}$.

Hence using a lower discount factor is equivalent to using $\gamma$, the true value of the discount factor for the MDP, and replacing the transition matrix with its weighted average with a matrix of zeros. Applying this to our unified framework, we replace the MLE transition matrix for action $a$ with the regularized form:
\begin{equation}
\hat{T}(a)= (1-\epsilon)\hat{T}_{MLE}(a)+\epsilon T_{zeros}
\label{eq:discount}
\end{equation}
(Full derivation in Appendix~\ref{appx:disc_reg}.)

\paragraph{Planning over $\epsilon$-Greedy Policies}
Finally, we frame planning over the set of $\epsilon$-greedy policies as a weighted average transition matrix.  When finding the optimal policy from the estimated MDP by policy iteration, all policies are treated as $\epsilon$-greedy. Then we perform the greedy, deterministic policy that had the best $\epsilon$-greedy performance.

When following an $\epsilon$-greedy policy, for greedy action $a$, the agent transitions according to transition matrix $T(a)$ with probability $(1-\epsilon)$ and chooses uniformly at random between the transition matrices for all actions with probability $\epsilon$.  Estimating each transition matrix by its MLE, the transitions under an $\epsilon$-greedy policy corresponds to: 
\begin{equation}
 \hat{T}(a) = (1-\epsilon)\hat{T}_{MLE}(a) + \epsilon \frac{1}{|\mathcal{A}|} \sum_{a'} \hat{T}_{MLE}(a')
 \label{eq:greedy}
\end{equation}
Recall that we restrict our consideration to the case of state-dependent rewards $R(s)$. Under this assumption, planning over the set of $\epsilon$-greedy policies is equivalent to replacing the MLE transition matrix for each action with Equation \ref{eq:greedy} before computing the optimal greedy policy.

\section{Discussion of Unified Framework}
With the methods expressed in a common form, we can now make predictions about their relative performance.

\paragraph{Uniform Prior and Discount Regularization Connection}
A surprising result revealed by the unified form is that, when constrained to the weighted average form (uniform exploration and equal priors), a uniform prior produces the same optimal policy as discount regularization for the same value of $\epsilon$. 

\begin{thm}\label{thm:uniform_discount_conn}
Let $M_1$ and $M_2$ be finite-state MDPs with identical state space, action space, and reward function.  Let $M_1$ have transition function $T$ and discount factor $(1-\epsilon) \gamma$. Let $M_2$ have discount factor $\gamma$ and transition function $(1-\epsilon)T + \epsilon T_{unif}$, where $T_{unif}$ is the uniform transition matrix. Then $M_1$ and $M_2$ have the same optimal policy.
\end{thm}
\begin{proof}
    See Appendix~\ref{pf:unif_disc_equiv} for proof.
\end{proof}

\paragraph{Impact of MDP Structure}

\textit{A uniform Dirichlet prior is a good regularizer in a dense world.} With a uniform prior, the posterior transition matrix is not constrained by the connections between states in the true MDP. If the MDP has a high level of connectivity between states, the connectivity of a uniform prior is appropriate, however in the case of a sparsely connected MDP, assuming all states are linked is unlikely to be optimal.

\textit{Discount regularization balances between planning lengths.} The discount factor determines planning horizon, prioritizing shorter- versus longer-term rewards. The weighted average view of discount regularization is consistent with the view of discounting as causing the agent to act as if it transitions according to the true transition matrix with probability $1-\epsilon$ and exit the MDP (represented by the matrix of zeros) with probability $\epsilon$ \citep[p.~113]{suttonbarto}.  Faced with the prospect of exit, the agent prioritizes closer rewards.  We predict that this is beneficial when balancing the trade-offs of differently sized rewards at different distances.

\textit{$\epsilon$-greedy planning avoids catastrophic outcomes.} In the $\epsilon$-greedy case, averaging the transition matrices of all actions causes the agent to act as if there is more stochasticity in the transitions. We hypothesize that the added randomness during planning will cause the agent to find a more conservative policy and perform better in MDPs with  catastrophic outcomes.

\paragraph{Impact of Data Collection Policy}
In the unified form for discount regularization, the regularization matrix is the same for all state-action pairs. In contrast, the regularization matrix for planning over $\epsilon$-greedy policies is the same for all actions, but differ by state. Finally, a Dirichlet prior, when not constrained to uniform visits, regularizes each state-action pair separately. Therefore, for data sets with uneven counts across states and/or actions, we expect a Dirichlet prior to perform best, followed by $\epsilon$-greedy planning then discount regularization because of the ability to separately tailor the regularization to the state-action pair.

Furthermore, examining the equivalence between discount regularization and the weighted average form of the uniform Dirichlet prior reveals that discount regularization functions like a Dirichlet prior with all parameters of magnitude $\frac{\gamma-\gamma_l}{\gamma_l}\frac{\sum c_i}{N}$ (see Appendix~\ref{appx:dirich_disc_prior} for derivation). This underscores the limitations of discount regularization under uneven data collection. The magnitude of the prior is higher for state-action pairs with more data, which is not desirable.  

\section{Empirical Examples}
\begin{figure}[ht]
\centering
\includegraphics[width=.45\textwidth]{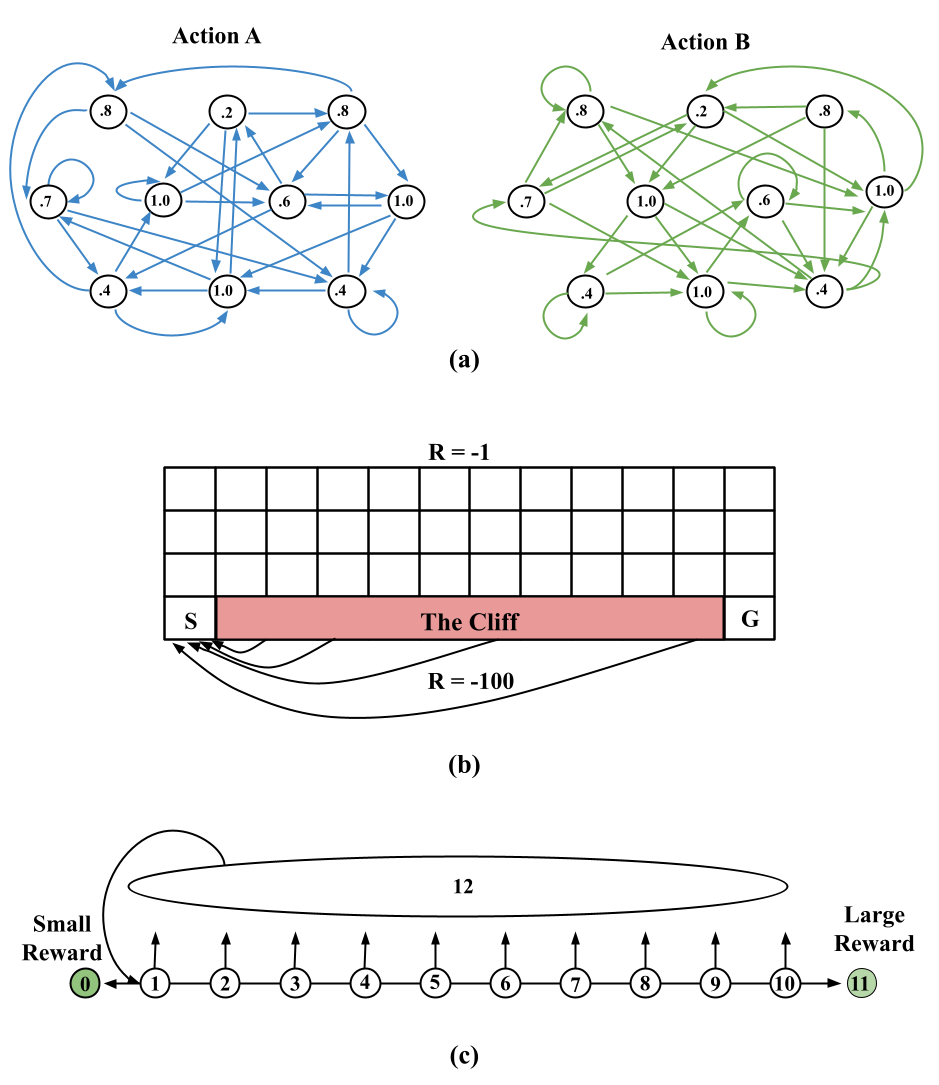}
\caption{(a) Interconnected Grid (b) Cliff Walk (c) Two Goals}
\label{fig:three_MDPs}
\end{figure}
Equipped with a common framework, we implement the three regularization methods across simple tabular examples. We explore the impact of following characteristics on the loss of the resulting policy: MDP structure, probability that an action in the data set is generated by the optimal policy, starting state of trajectories in the data set, and data set size.

\subsection{MDP Types} 

\begin{figure*}[h!]
\centering
\includegraphics[width=\textwidth]{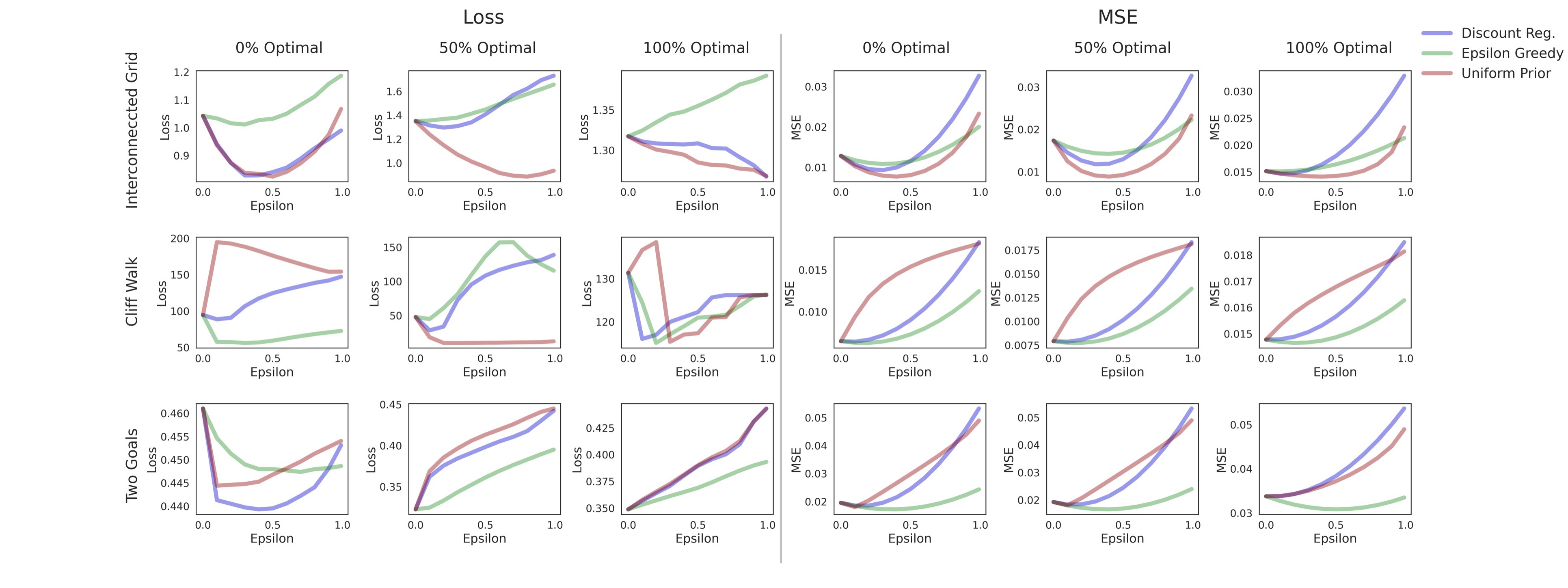}
\caption{Loss and MSE by Optimality of Data Collection Policy. Predictions by MDP hold for random data collection, but results vary when data partially- or fully-generated from optimal policy. Random start states; 15 trajectories of length 10 each for Interconnected Grid and Two Goals; 25 trajectories of length 20 for Cliff Walk. }
\label{fig:loss_mse}
\end{figure*}

\paragraph{Dense world: Interconnected Grid} This example MDP illustrates dense and complex connections between states.  For each action, the agent transitions according to the arrows in Figure \ref{fig:three_MDPs}(a) with equal probability. Rewards are normally distributed with means as indicated.

\paragraph{Catastrophic Outcome: Cliff Walk} The Cliff Walk example from \citet{suttonbarto} represents an agent that moves left, right, up, and down, with added noise. Mean rewards are -100 in the cliff and -1 for all other transitions. After reaching state $\textbf{G}$, the agent receives no further rewards.

\paragraph{Different Planning Lengths: Two Goals} The final example MDP depicts differently sized rewards on opposite ends of a linear grid. The agent moves left, right, or up, each with noise. Transitioning to state 0 results in a reward of mean 0.10 and transitioning to state 11 results in a reward of mean 1. Rewards for all other transitions have mean 0 and after reaching either the small or large reward, the agent receives no further rewards.

\subsection{Implementation}
For each MDP, for a range of data collection policies and sizes, we generate trajectories from the true MDP. We estimate the transition and reward matrices as the MLE. For a range of $\epsilon$ between 0 and 1 (or prior magnitude from 0 to 1000), we regularize the transition matrix for each action. We find the optimal policy via policy iteration. The policy that is optimal in the estimated, regularized MDP and the optimal policy for the true MDP are compared in terms of performance on the true MDP. Details are in Appendix~\ref{appx:imp_deets}.

\subsection{Results}
To compare regularization methods across MDPs and data sets, we plot the loss of the resulting policy. We also plot mean squared error (MSE) of the estimated, regularized transition matrix compared to the true transition matrix to investigate the extent to which better approximating the true transition matrix drives lower loss. 

\paragraph{Hypothesized interactions between MDP structure and regularization are confirmed, although mediated by data collection policy.}
In Figure~\ref{fig:loss_mse}, comparing results by MDP confirms our predictions under the condition of uniformly random data collection, although less pronounced in the case of Interconnected Grid. Deviating from a uniformly random data collection policy by generating an increasing percentage of the actions from the true optimal policy impacts which method minimizes loss as well as whether regularization is beneficial at all. We note a reversal in which regularization method minimizes loss for the Cliff Walk when data is not fully random, and with Two Goals, no regularization method is beneficial as actions in the data set are generated increasingly from the optimal policy. 

\paragraph{Data set size and starting state demonstrate less impact on relative loss.} Number and length of trajectories do not considerably impact relative performance. The impact of trajectory starting state varies by MDP. While the ordering of methods by loss is not dramatically shifted across starting states, the shape of the loss curves indicate a differing impact of regularization. Results are in Appendices~\ref{results_data_size} and \ref{results_start}.

\paragraph{Impact of uneven data collection is inconclusive.} We hypothesized that regularizers that are more flexible in allowing different amounts of regularization across state-action pairs outperform under uneven exploration. Both generating the data set from the optimal policy and restricting starting state cause uneven exploration, yet we do not clearly observe the hypothesized relationship. Further investigation is needed to isolate the impact of uneven exploration. 

\paragraph{Lower loss does not consistently correspond to lower transition matrix MSE.} Although loss is partially driven by the ability to accurately replicate the true transition matrix, there are other factors impacting loss to be identified.

\vspace{.5cm}
To summarize, choosing between regularization methods can be viewed in terms of choosing the regularization matrix from the weighted average form that best aligns with the context at hand in terms of both the data collection policy and MDP structure. The empirical examples in this section provide evidence that both of these factors impact relative loss. Further work remains to formalize the conditions in which each regularizer is preferred.

\section{Conclusion}
We have unified three MDP regularization methods into a common framework that helps us to predict and understand their performance in different settings. The common form and empirical examples demonstrate that it is vital to consider both the the MDP structure and the data collection policy when deciding between regularization methods. The unified form also motivates viewing discount regularization as replacing the maximum likelihood estimate transition matrix with its posterior mean under a uniform prior, when the data set is constrained to uniform state visitation. In future work, we will leverage the unified framework to systematically characterize MDPs and data sets to select a regularization method.

\vspace{.5cm}
\textbf{Acknowledgements}

Research reported in this work was supported by the National Institute Of Biomedical Imaging And Bioengineering and the Office of the Director of the National Institutes of Health under award number P41EB028242, the National Institute on Alcohol Abuse and Alcoholism under award number  R01AA023187. 

SR and FDV acknowledge support from NSF project 2007076.  Research reported in this work was also supported by the National Institute Of Biomedical Imaging And Bioengineering and the Office of the Director of the National Institutes of Health under award number P41EB028242, the National Institute on Alcohol Abuse and Alcoholism under award number  R01AA023187. 

The content is solely the responsibility of the authors and does not necessarily represent the official views of the National Institutes of Health.

\bibliography{main}
\bibliographystyle{icml2021}

\onecolumn
\appendix
\section*{APPENDIX}

\section{Full Derivation of Unified Form}

\subsection{Dirichlet Prior}\label{appx:dirchlet}

Assume prior $P_{prior}(T(s,a)) = \text{Dirichlet} (\langle \alpha_1,...,\alpha_N \rangle )$ on transition matrix $T(s,a)$ and let $\langle c_1,...,c_N \rangle$ be the transition count data observed from state $s$ to states 1 to $N$ under action $a$. The posterior of $T(s,a)$ follows a Dirichlet distribution with parameter $\langle c_1 + \alpha_1,...,c_N + \alpha_N \rangle$ and the posterior mean is:
\begin{align*}
T_{\text{post mean}}(s,a)= \langle \frac{c_1+\alpha_1}{\sum c_i + \sum \alpha_i},...,\frac{c_N+\alpha_N}{\sum c_i + \sum \alpha_i} \rangle
\end{align*}
\begin{align*}
& T_{\text{post mean}}(s,a) = \langle \frac{c_1}{\sum c_i + \sum \alpha_i},...,\frac{c_N}{\sum c_i + \sum \alpha_i} \rangle + \langle \frac{\alpha_1}{\sum c_i + \sum \alpha_i},...,\frac{\alpha_N}{\sum c_i + \sum \alpha_i} \rangle
\end{align*}
Multiply each term by 1.
\begin{align*}
T_{\text{post mean}}(s,a) &= \frac{\sum c_i}{\sum c_i} \langle \frac{c_1}{\sum c_i + \sum \alpha_i},...,\frac{c_N}{\sum c_i + \sum \alpha_i} \rangle +\frac{\sum \alpha_i}{\sum \alpha_i} \langle \frac{\alpha_1}{\sum c_i + \sum \alpha_i},...,\frac{\alpha_N}{\sum c_i + \sum \alpha_i} \rangle \\
&= \frac{\sum c_i}{\sum c_i + \sum \alpha_i} \langle \frac{c_1}{\sum c_i},...,\frac{c_N}{\sum c_i} \rangle +\frac{\sum \alpha_i}{\sum c_i + \sum \alpha_i} \langle \frac{\alpha_1}{\sum \alpha_i},...,\frac{\alpha_N}{\sum \alpha_i} \rangle
\end{align*}

Let $\hat{T}_{MLE}(s,a)$ be the MLE of $T(s,a)$: $\hat{T}_{MLE}(s,a)= \langle \frac{c_1}{\sum c_i},...,\frac{c_N}{\sum c_i} \rangle$.

Let $T_{\text{prior mean}}(s,a)$ be the transition matrix implied by the prior for state $s$ and action $a$. $T_{\text{prior mean}}(s,a)= \langle \frac{\alpha_1}{\sum \alpha_i},...,\frac{\alpha_N}{\sum \alpha_i} \rangle$.

Using $\hat{T}_{MLE}(s,a)$ and $T_{\text{prior mean}}(s,a)$, we can write $T_{\text{post mean}}(s,a)$ as follows.
\begin{align*}
T_{\text{post mean}}&(s,a) = \frac{\sum c_i}{\sum c_i + \sum \alpha_i} \hat{T}_{MLE}(s,a) + \frac{\sum \alpha_i}{\sum c_i + \sum \alpha_i }T_{\text{prior mean}}(s,a)
\end{align*}
Let $\epsilon=\frac{\sum \alpha_i}{\sum c_i + \sum \alpha_i}$. Consequently, we have:
\begin{align*}\label{eq:1}
T_{\text{post mean}}(s,a)=(1-\epsilon) \hat{T}_{MLE}(s,a) + \epsilon T_{\text{prior mean}}(s,a)
\end{align*}

The expression above is for a single state $s$. To write $T_{\text{post mean}}(a)$ as a matrix for all states for a given action, we must pull out the same factor of $\epsilon$ and $(1-\epsilon)$ for all states.  Hence we assume:
\setlist{nolistsep}\begin{enumerate}[noitemsep]
    \item $\sum c_i$ equal across all states for a given action $a$, and 
    \item $\sum \alpha_i$ equal across all states for a given action $a$,
\end{enumerate}

Assuming the conditions above and taking the matrix of posterior means as our estimate of $T(a)$:
\begin{equation*}
\hat{T}(a)=(1-\epsilon) \hat{T}_{MLE}(a) + \epsilon T_{\text{prior mean}}(a)
\end{equation*}

\subsection{Discount Regularization}\label{appx:disc_reg}

Consider the matrix form of the Bellman equation, using $\gamma_l < \gamma$, the lower value of the discount factor used for regularization: $V = R + \gamma_l T V$.  By the steps below, we write the product $\gamma_l T$ from the Bellman equation as the product of true discount factor $\gamma$ and a weighted average matrix.

First add and subtract $\gamma$.
\begin{align*}
\gamma_l T=[\gamma - (\gamma - \gamma_l)]T
\end{align*}
Pull out a factor of $\gamma$.
\begin{align*}
    \gamma_l T= \gamma(1 - \frac{(\gamma - \gamma_l)}{\gamma})T
\end{align*}
Let $T_{zeros}$ be an appropriately sized matrix of zeros. Adding $\gamma T_{zeros}$ to the right hand side does not change the equality.
\begin{align*}
\gamma_l T = \gamma[(1-\frac{\gamma-\gamma_l}{\gamma})T+T_{zeros}]
\end{align*}
Multiply the $T_{zeros}$ term inside the parentheses by $\frac{\gamma - \gamma_l}{\gamma}$. $T_{zeros}$ is all zeros so a multiplier does not affect the equality.
\begin{align*}
\gamma_l T = \gamma[(1-\frac{\gamma-\gamma_l}{\gamma})T+(\frac{\gamma-\gamma_l}{\gamma})T_{zeros}]
\end{align*}
Let $\epsilon=\frac{\gamma-\gamma_l}{\gamma}$, 
\begin{align*}
\gamma_l T = \gamma[(1-\epsilon) T_{true} + \epsilon T_{zeros}]
\end{align*}
We have replaced the product of the regularization discount factor and the true transition matrix with the product of the true discount factor and a weighted average of the transition matrix and a matrix of zeros.  To put this in the unified framework, consider regularizing the MLE transition matrix for action $a$ via discount regularization. Using the proof in this section, our regularized estimated transition matrix for action $a$, $\hat{T}(a)$, is:
\begin{equation*}
\hat{T}(a)= (1-\epsilon)\hat{T}_{MLE}(a)+\epsilon T_{zeros}
\label{eq:discount}
\end{equation*}

\subsection{Discount Regularization - Uniform Prior Connection}

\subsubsection{Proof of Theorem 1: Equivalence of Weighted Average Form}\label{pf:unif_disc_equiv}

\paragraph{First we show that the optimal policy is not affected by adding the same constant $x$ to all rewards $r(s,a)$.}  Let $Q^{\pi}_x(s,a)$ be the action-value function for policy $\pi$ when adding constant $x$ to all rewards. Then, 

$Q^{\pi}_x(s,a)= \mathbb{E}_{\pi} [\sum_{k \geq 0}^{\infty} \gamma^k (r(s_k,a_k) + x) | s_0 = s, a_o = a] = \mathbb{E}_{\pi} [\sum_{k \geq 0}^{\infty} \gamma^k r(s_k,a_k) | s_0 = s, a_o = a] + \frac{x}{1-\gamma}$

and the action-value function of the optimal policy is, 
$Q_x^*(s,a)= \text{max}_{\pi}[ \mathbb{E}_{\pi} [\sum_{k \geq 0}^{\infty} \gamma^k r(s_k,a_k) | s_0 = s, a_o = a] + \frac{x}{1-\gamma}]$

The optimal action at state $s$ is $\pi_{opt}(s) = \text{argmax}_{a} Q^*_x(s,a)$. The first term of the expression for $Q^*_x(s,a)$ does not contain $x$ and the second does not depend on $a$, therefore $\pi_{opt}$ is not affected by the choice of any constant added to $r(s,a)$.

\paragraph{Next, observe that $Q_x^*(s,a)$ is the solution to Bellman's optimality equation,} 
$Q_x(s,a) = r(s,a) + x + \gamma \sum_{s'} T(s,a,s') \text{max}_{a'} Q_x(s',a')$. From above, we established that $\pi_{opt} = \text{argmax}_{a} Q^*_x(s,a)$ does not depend on $x$. Therefore the solution to Bellman's optimality equation also does not depend on $x$.

\paragraph{Bellman's optimality equation for a transition matrix regularized by averaging with the uniform transition matrix can be written in terms of a scaled discount factor and added constant.} In this case, Bellman's optimality equation is  $Q^*(s,a) = r(s,a) + \gamma \sum_{s'}[ ( (1 - \epsilon) T(s,a,s') + \epsilon \frac{1}{n}) \text{max}_{a'} Q^*(s',a')]$ \\
$Q^*(s,a) = r(s,a) + \gamma (1 - \epsilon) \sum_{s'} T(s,a,s') \text{max}_{a'} Q^*(s',a') + \gamma \frac{\epsilon}{n} \sum_{s'} \text{max}_{a'} Q^*(s',a')$ \\
Letting $x=\gamma \frac{\epsilon}{n} \sum_{s'} \text{max}_{a'} Q^*(s',a')$, Bellman's optimality equation is:\\ 
$Q^*(s,a) = r(s,a) + x + \gamma (1 - \epsilon) \sum_{s'} T(s,a,s') \text{max}_{a'} Q^*(s',a')$

\paragraph{$x$ is constant with respect to $a$, so by this first section of the proof, it does not affect the optimal policy.}  Therefore we can write the expression for the optimal policy at state $s$ as: \\
$\pi_{opt}(s) = \text{argmax}_a  Q^*(s,a)$ \\
$\pi_{opt}(s)=\text{argmax}_a ( r(s,a) + x + \gamma (1 - \epsilon) \sum_{s'} T(s,a,s') \text{max}_{a'} Q^*(s',a'))$ \\
$\pi_{opt}(s)=\text{argmax}_a ( r(s,a) + \gamma (1 - \epsilon) \sum_{s'} T(s,a,s') \text{max}_{a'} Q^*(s',a'))$

This is the optimal policy for the MDP with the original transition matrix and discount factor $(1- \epsilon ) \gamma$. This is equivalent to discount regularization, and matches the value of epsilon $\epsilon = \frac{\gamma - \gamma_l}{\gamma}$ that we derived in the previous section.

\subsubsection{Dirichlet Prior Implied by Discount Regularization}\label{appx:dirich_disc_prior}
The equivalence proof above demonstrates that, for a given value of $\epsilon$, averaging the transition matrix with the uniform matrix or with the matrix of zeros yields the same policy.  Averaging with the uniform matrix is only exactly equivalent to a uniform prior if the sum of the transition counts is equal for all starting states, for a given action. Recall that $\epsilon = \frac{\sum \alpha_i}{\sum \alpha_i + \sum c_i}$, where $c_i$ are the transition counts from the data and $\alpha_i$ are the parameters of the Dirichlet prior. We can solve to find the prior magnitude $\alpha_i$ implied by the choice of $\epsilon$ and observed transition counts $\sum c_i$ in the weighted average form.  This reveals what Dirichlet prior we are implicitly using when we regularize by the weighted average uniform form, and consequently the Dirichlet prior implied by discount regularization.

From $\epsilon=\frac{\sum \alpha_i}{\sum \alpha_i + \sum c_i}$, observe $\sum \alpha_i = \frac{\epsilon}{1-\epsilon}\sum c_i$.  We assume $N$ states. For the uniform distribution, all $\alpha_i$ for a given state are the same, so substitute $\sum \alpha_i = N \alpha_i$ to get $\alpha_i =\frac{\epsilon}{1-\epsilon}\frac{\sum c_i}{N}$.  Therefore, using a lower discount rate yields the same optimal policy as setting a uniform Dirichlet prior over each row of the transition matrix with magnitude $\frac{\epsilon}{1-\epsilon}\frac{\sum c_i}{N}$. 

We can relate this back to the value of $\gamma$.  Recall $\epsilon = \frac{\gamma - \gamma_l}{\gamma}$, where $\gamma$ is the true value of the discount factor and $\gamma_l$ is the lower value used for regularization. Plugging this into the expression above yields  $\alpha_i = \frac{\gamma-\gamma_l}{\gamma_l}\frac{\sum c_i}{N}$. So discount regularization functions like a Dirichlet prior 
\begin{equation}
T_{prior}(s,a) \sim \text{Dirichlet}(\frac{\gamma-\gamma_l}{\gamma_l}\frac{\sum c_i}{N},...,\frac{\gamma-\gamma_l}{\gamma_l}\frac{\sum c_i}{N})
\label{eqn:disc_reg}
\end{equation}
where again $\sum c_i$ is the total number of transitions in the data starting at state $s$.

\section{Implementation Details}\label{appx:imp_deets}

\begin{algorithm}[H]
   \caption{Regularization Loss Pseudocode}
   \label{alg:example}
\begin{algorithmic}
   \STATE {\bfseries Input:} \text{MDP, epsilon list, regularization method}
   \FOR{$i=1$ {\bfseries to} 5000}
   \STATE Generate data set: $n$ trajectories of length $l$
   \STATE Estimate MDP from data
    \FOR{$\epsilon$ in epsilon list}
    \STATE Regularize transition matrices by amount $\epsilon$
    \STATE Calculate optimal policy $\pi$ of regularized MDP
    \STATE Calculate loss comparing $\pi$ vs. true optimal policy in true MDP
    \ENDFOR
   \ENDFOR
    \STATE Average loss by $\epsilon$ value across all data sets
\end{algorithmic}
\end{algorithm}

To compare policies resulting from different regularization methods, we implement the following procedure, summarized in Algorithm \ref{alg:example}. Separately for each of the three example MDPs, we repeatedly generate data sets of trajectories from the true MDP. We estimate the transition and reward matrices as the MLE of the data. The estimate of the reward function at state-action pair $(s,a)$ is then the mean of the observed rewards at $(s,a)$.  The estimated probability of transition from state $s$ to state $s'$ given action $a$ is the number of times the transition from $s$ to $s'$ is observed given action $a$ divided by the number of times state-action pair $(s,a)$ is observed in the data set. For state-action pairs that are not observed in the data, we assume equal transition probabilities to all states and reward of 0.50.

For each of a list of values of $\epsilon$ between 0 and 1 (or for uniform prior, a list of multipliers between 0 and 1000), we regularize the estimated transition matrix for each action. We then find the optimal policy via policy iteration.  Separately, we calculate the true optimal policy using the known, true MDP.  We then evaluate the policy found from the estimated, regularized MDP and the policy from the true MDP, both in the true MDP. We compute loss as the weighted average difference in values of the two policies across all states, weighted by the starting state distribution.

To explore the impact on different aspects of the data set, we vary the trajectory starting state, the length and number of trajectories, and the probability of an action being generated from the optimal policy versus a random policy. For the Cliff Walk, trajectory starting states considered are uniformly random, start at $S$, or start within 2 states of $G$. For the Interconnected Grid, starting states are either uniformly random, limited to 5 of the 10 states, or limited to 1 state.  In the case of Two Goals, starting states are uniformly random, starting in state 1 (next to the small reward) or starting in state 10 (next to the large reward).

The MSE is the squared difference in transition probabilities between the true and estimated transition matrices, averaged across all state-action pairs. For discount regularization, the weighted average form is not a true transition matrix. There is an implicit absorbing state that the agent enters with probability $\epsilon$ at each step. For discount regularization, we calculate the MSE in relation to the augmented transition matrix with the absorbing state and also without it.

\section{Additional Results}\label{appx:addl_results}

In the case of discount regularization, the regularized form is not a true transition matrix, so we also plot its MSE taking into account the implicit absorbing state.  We display plots of the MSE without the absorbing state as well because the scale allows for viewing the differences in detail.

\subsection{Results by Distance from Optimal}
\begin{figure}[!htp]
\centering
\includegraphics[width=\textwidth]{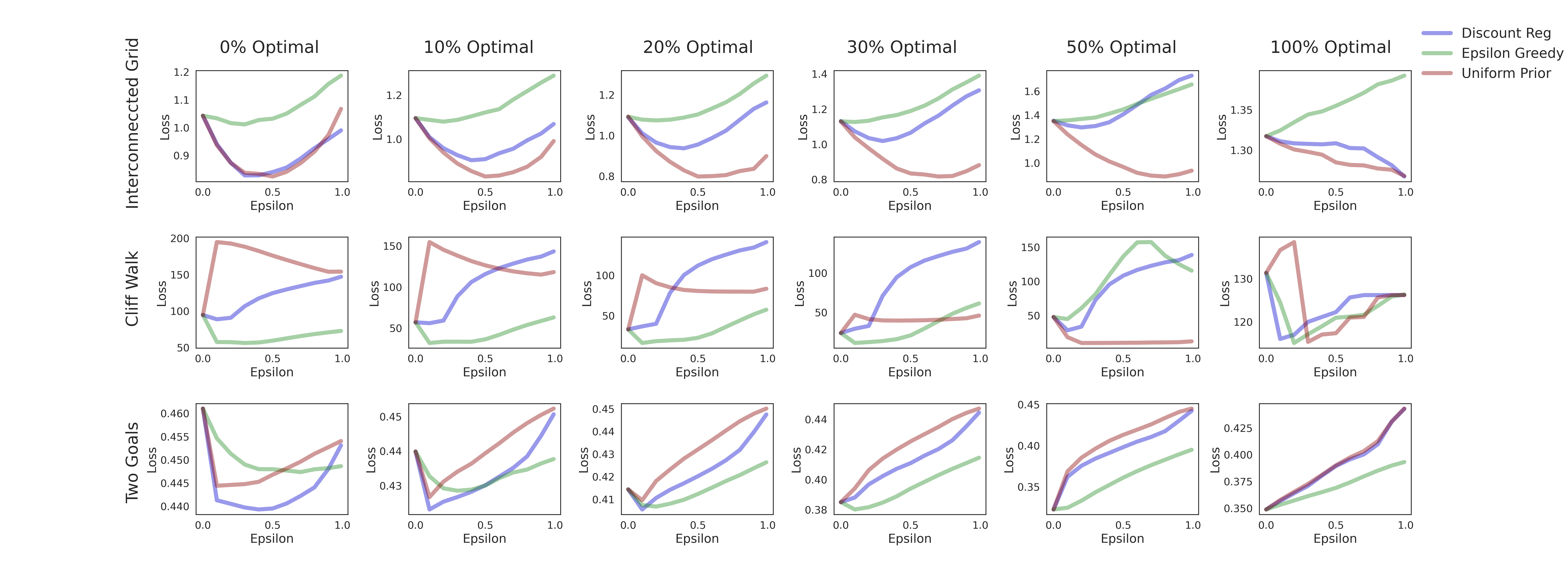}
\caption{\textbf{Loss, varying probability that actions in the data set are drawn from optimal policy.} Random start states; 15 trajectories of length 10 each for Interconnected Grid and Two Goals; 25 trajectories of length 20 for Cliff Walk. Percentages chosen to show change in shape of curve.}
\label{fig:loss_opt}
\end{figure}

\begin{figure}[!htp]
\centering
\includegraphics[width=\textwidth]{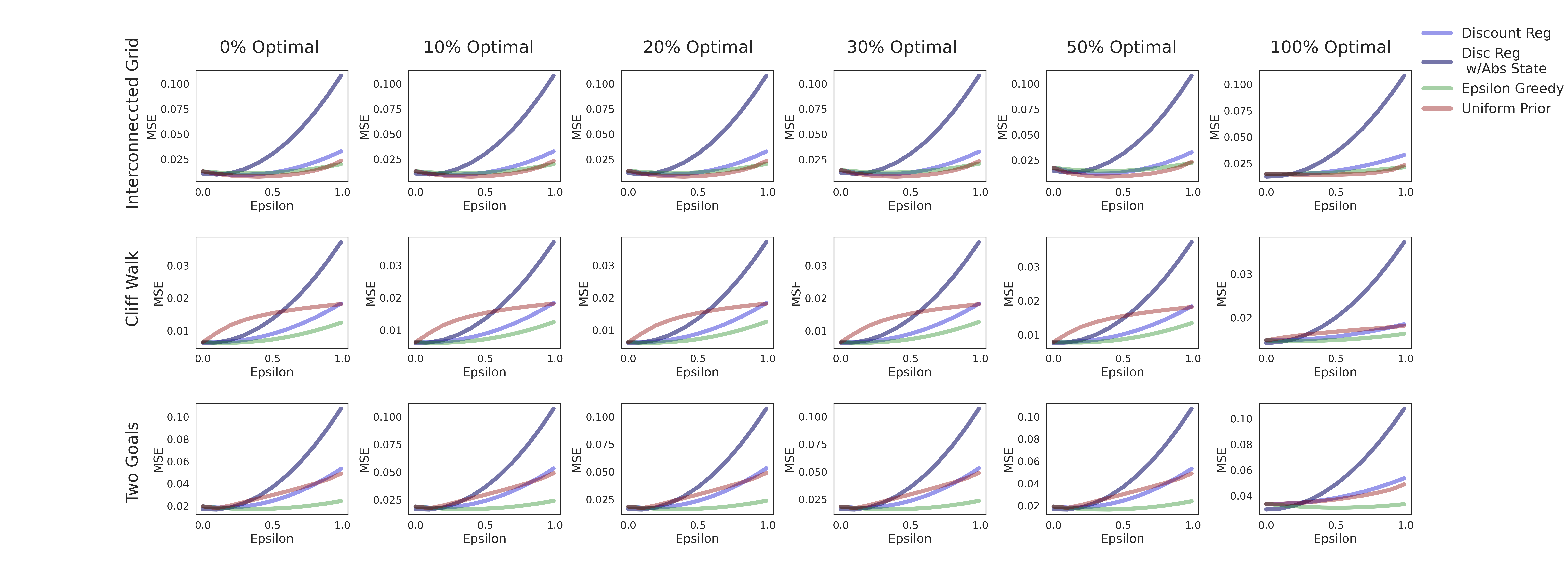}
\caption{\textbf{MSE, varying probability that actions in the data set are drawn from optimal policy.} Random start states; 15 trajectories of length 10 each for Interconnected Grid and Two Goals; 25 trajectories of length 20 for Cliff Walk. Percentages chosen to show change in shape of curve.}
\label{fig:mse_opt}
\end{figure}

\begin{figure}[!htp]
\centering
\includegraphics[width=\textwidth]{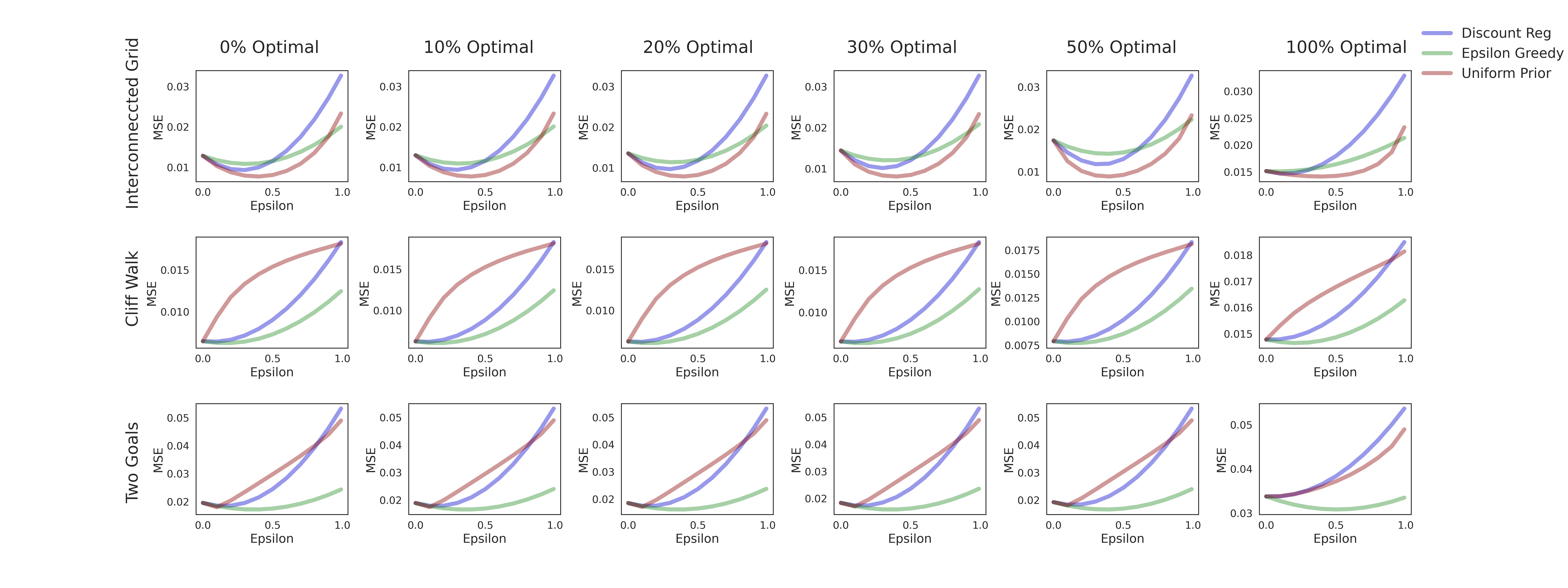}
\caption{\textbf{MSE, varying probability that actions in the data set are drawn from optimal policy.} Random start states; 15 trajectories of length 10 each for Interconnected Grid and Two Goals; 25 trajectories of length 20 for Cliff Walk. Percentages chosen to show change in shape of curve. Discount regularization with absorbing state removed for scale, to show detail on other curves.}
\label{fig:mse_opt2}
\end{figure}

\subsection{Results by Data Set Size}\label{results_data_size}
\subsubsection{Interconnected Grid}
\begin{figure}[H]
\centering
\includegraphics[width=.7\textwidth]{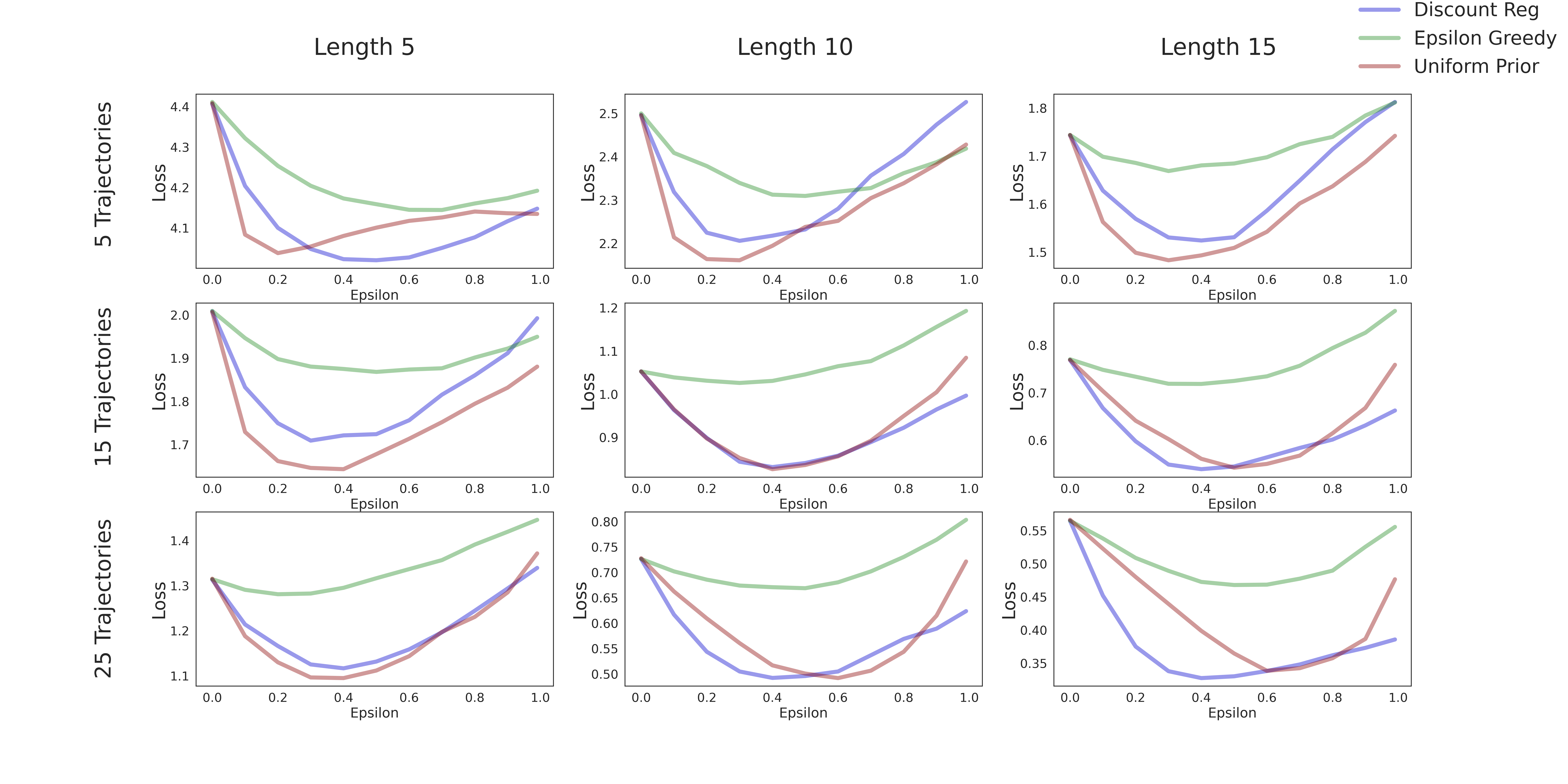}
\vspace{-.5cm}
\caption{\textbf{Interconnected Grid Loss} varying number and length of trajectories in data set. Random start states, random policy.}
\vspace{.4cm}
\includegraphics[width=.7\textwidth]{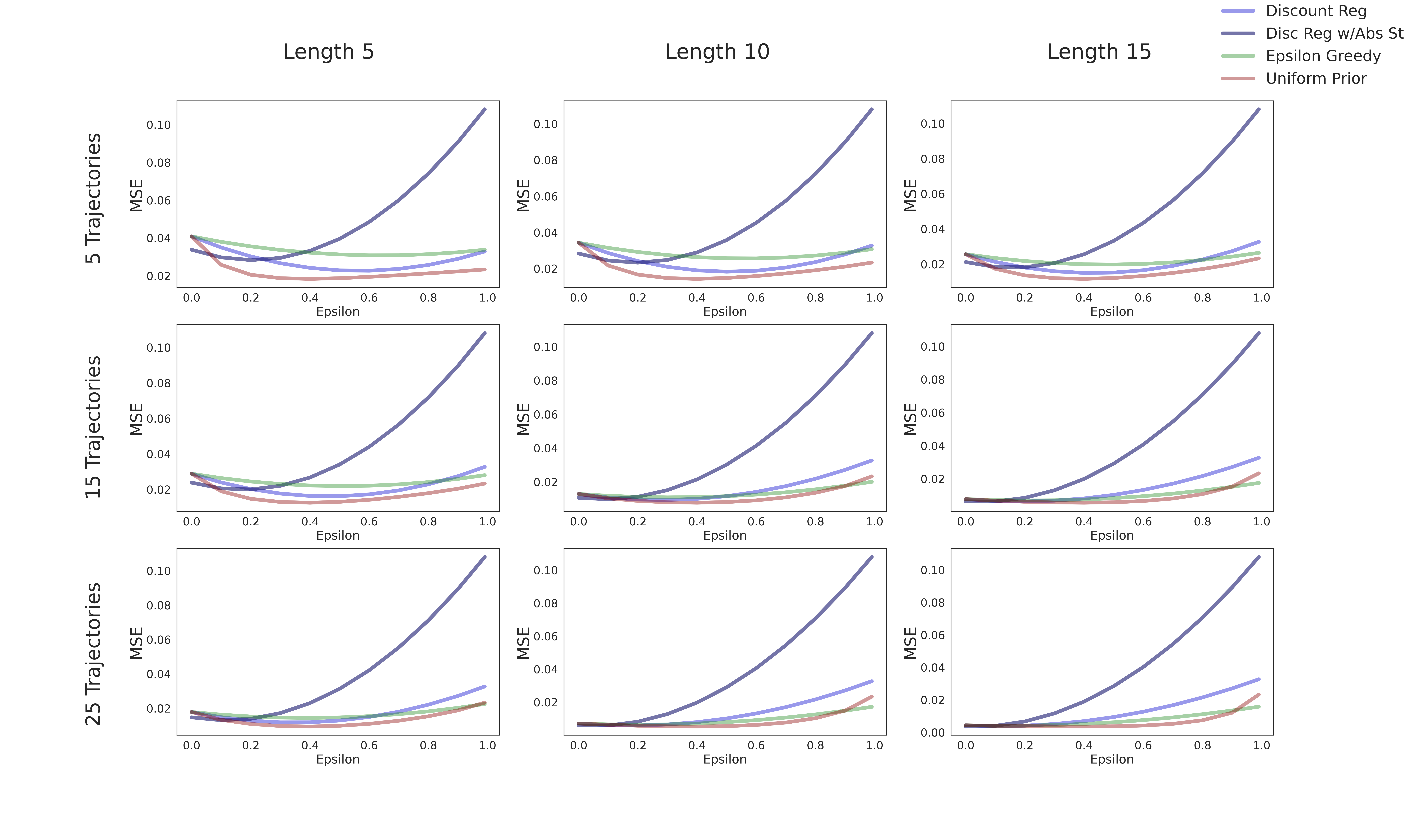}
\vspace{-.5cm}
\caption{\textbf{Interconnected Grid MSE} varying number and length of trajectories in data set. Random start states, random policy.}
\vspace{.4cm}
\includegraphics[width=.7\textwidth]{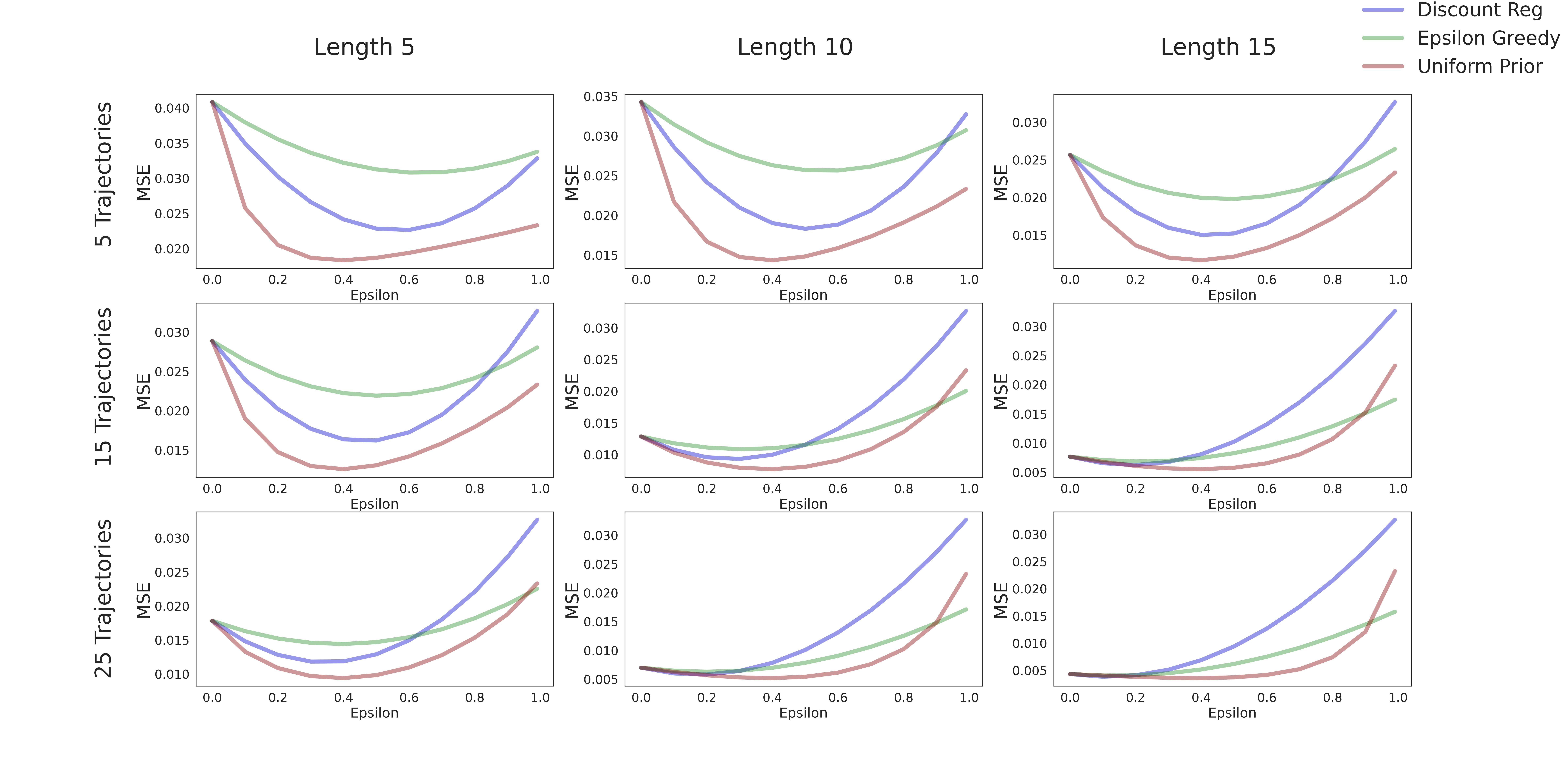}
\vspace{-.5cm}
\caption{\textbf{Interconnected Grid MSE} varying number and length of trajectories in data set. Random start states, random policy. Discount regularization with absorbing state removed for scale, to show detail on other curves.}
\label{fig:data_amt_IG}
\end{figure}

\subsubsection{Cliff Walk}
\begin{figure}[H]
\centering
\includegraphics[width=.7\textwidth]{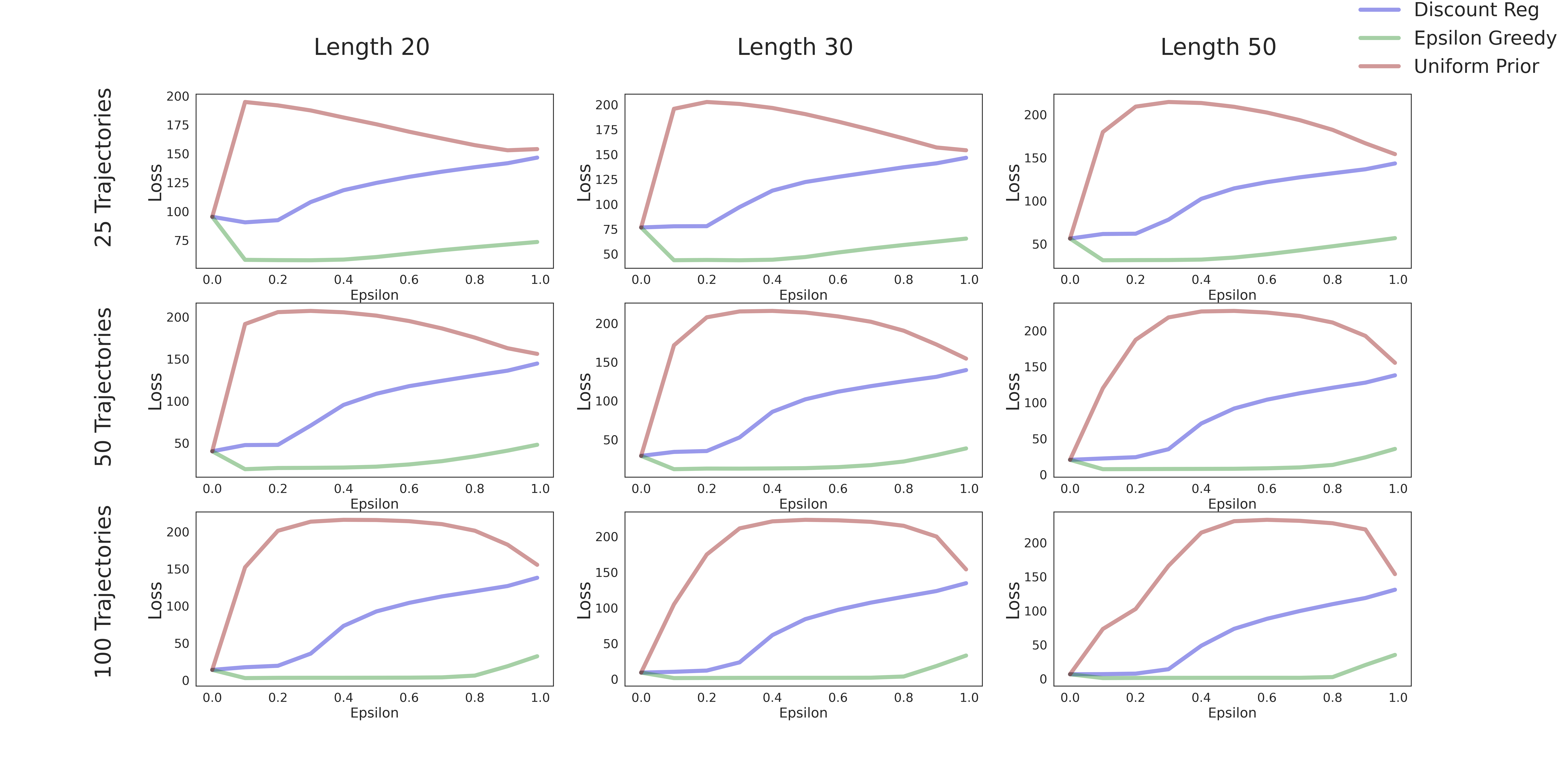}
\vspace{-.5cm}
\caption{\textbf{Cliff Walk Loss} varying number and length of trajectories in data set. Random start states, random policy.}
\vspace{.5cm}
\includegraphics[width=.7\textwidth]{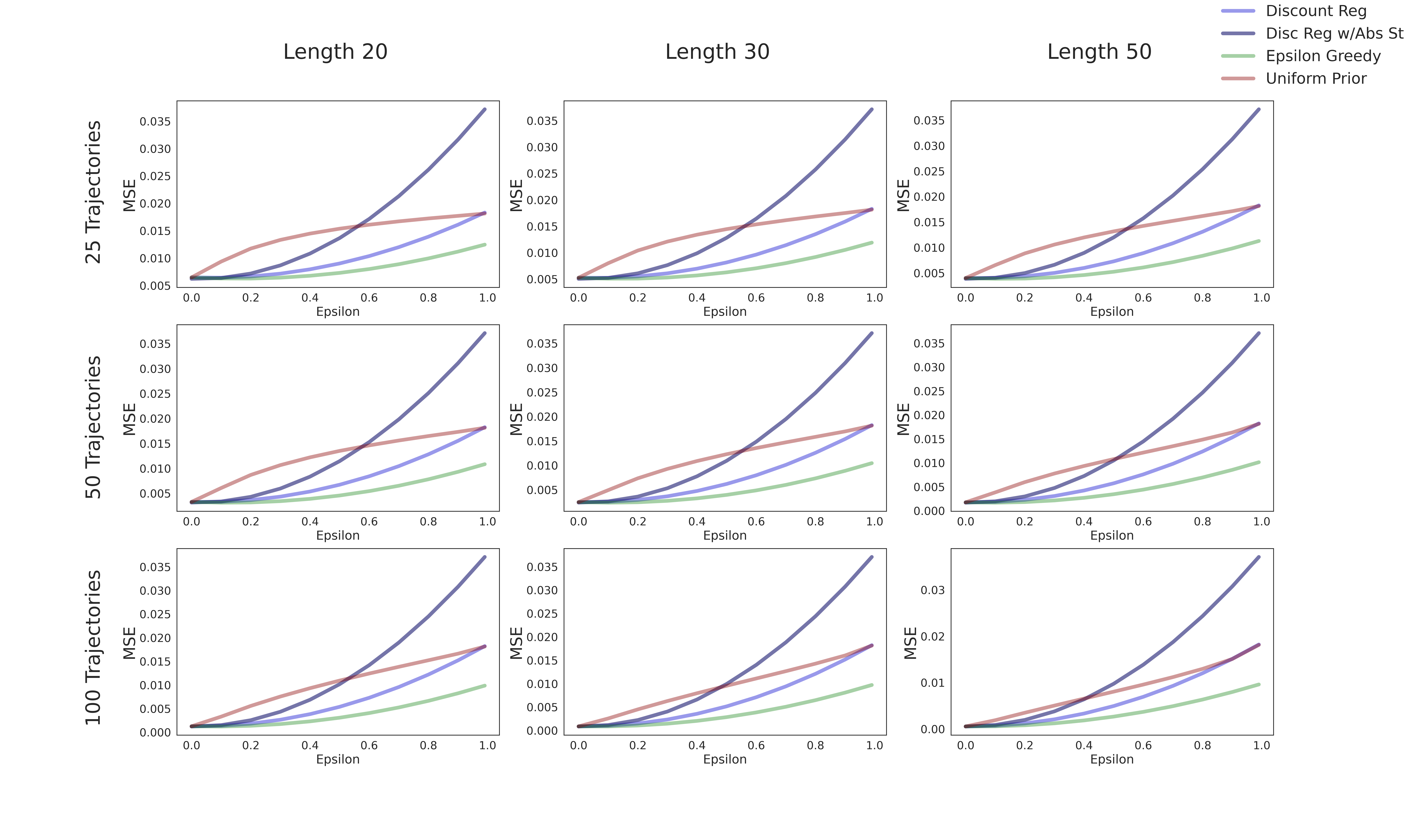}
\vspace{-.5cm}
\caption{\textbf{Cliff Walk MSE} varying number and length of trajectories in data set. Random start states, random policy.}
\vspace{.5cm}
\includegraphics[width=.7\textwidth]{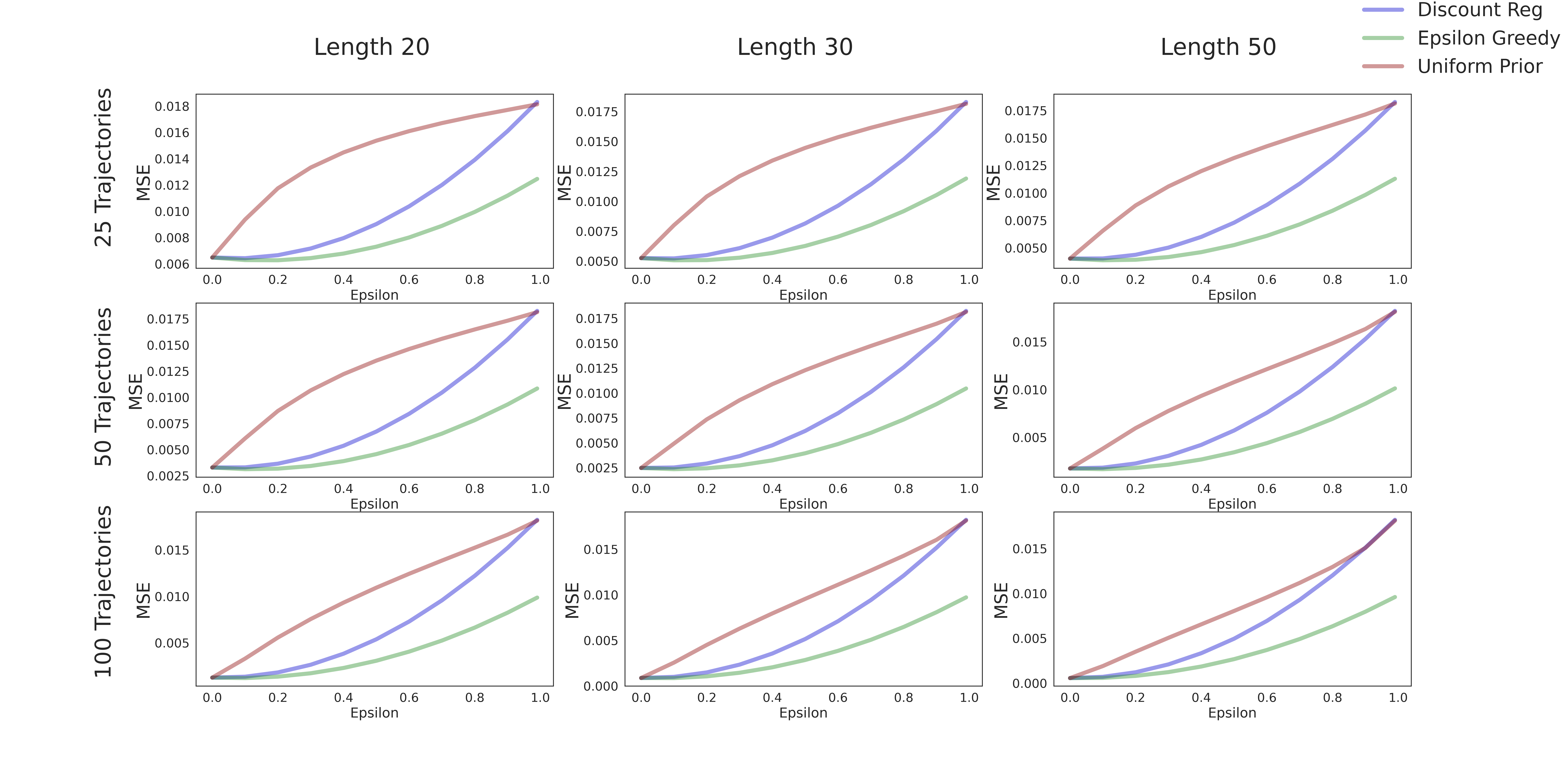}
\vspace{-.5cm}
\caption{\textbf{Cliff Walk MSE} varying number and length of trajectories in data set. Random start states, random policy. Discount regularization with absorbing state removed for scale, to show detail on other curves.}
\label{fig:data_amt_CW}
\end{figure}

\subsubsection{Two Goals}
\begin{figure}[H]
\centering
\includegraphics[width=.7\textwidth]{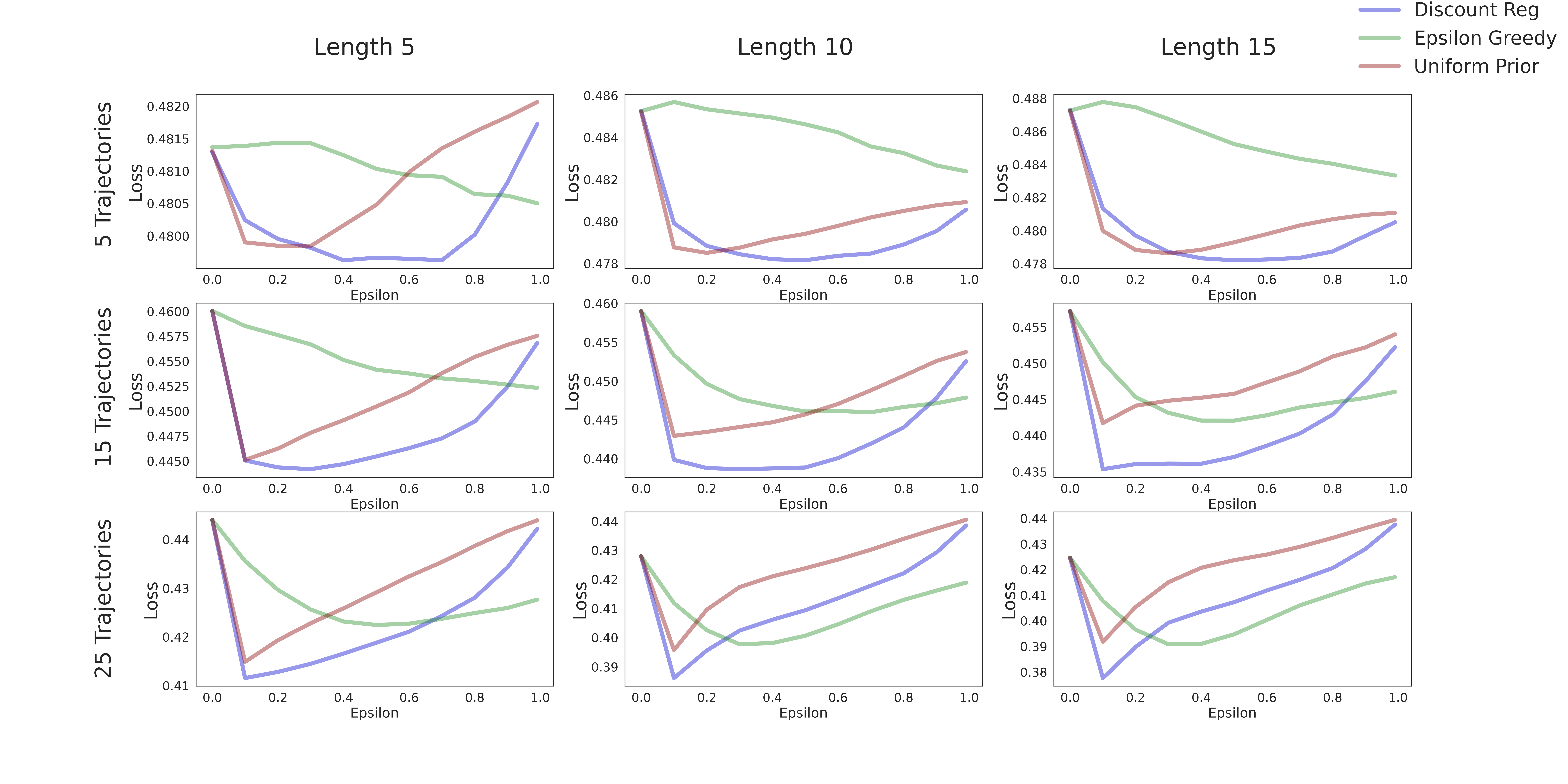}
\vspace{-.5cm}
\caption{\textbf{Two Goals Loss} varying number and length of trajectories in data set. Random start states, random policy.}
\vspace{.5cm}
\includegraphics[width=.7\textwidth]{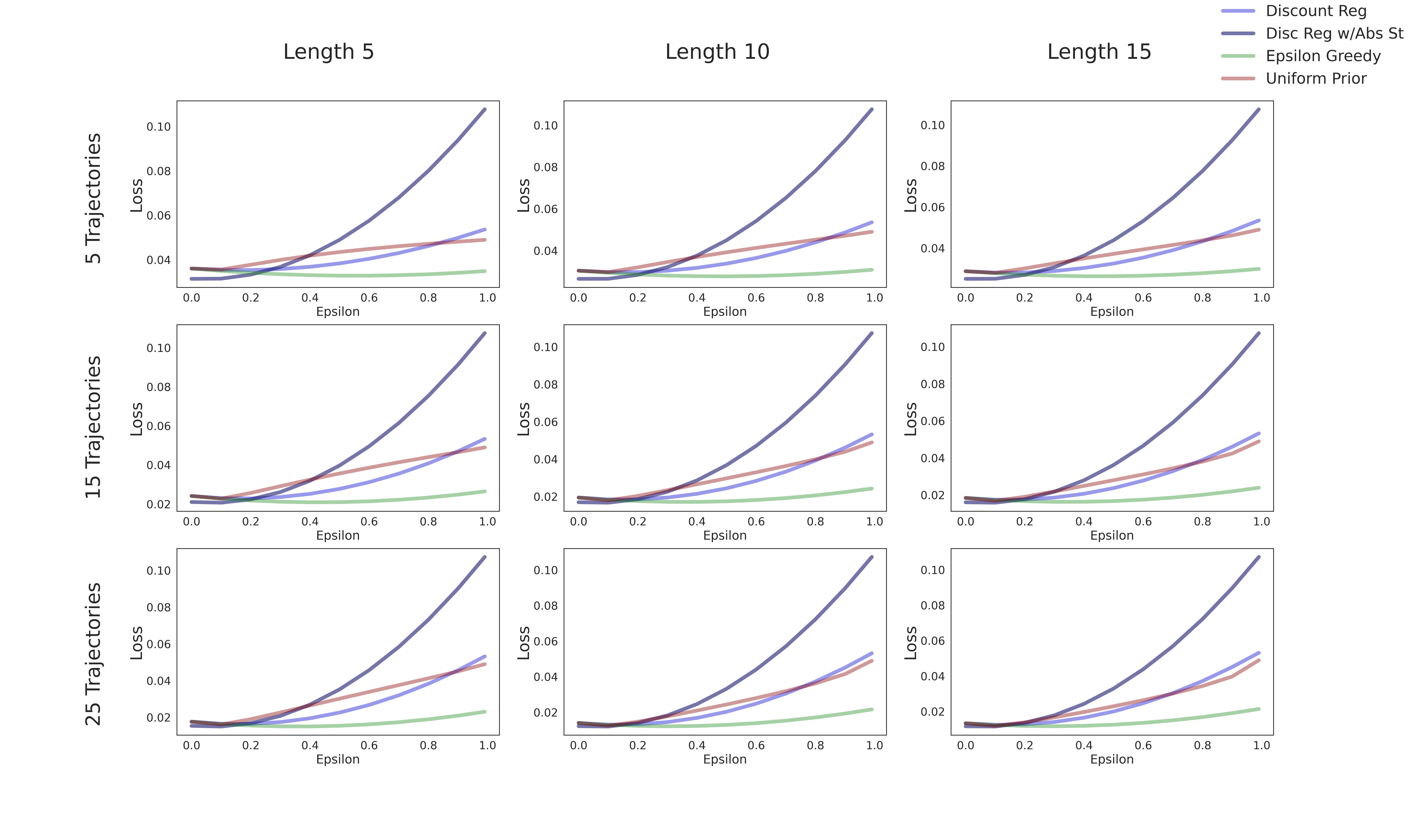}
\vspace{-.5cm}
\caption{\textbf{Two Goals MSE} varying number and length of trajectories in data set. Random start states, random policy.}
\vspace{.5cm}
\includegraphics[width=.7\textwidth]{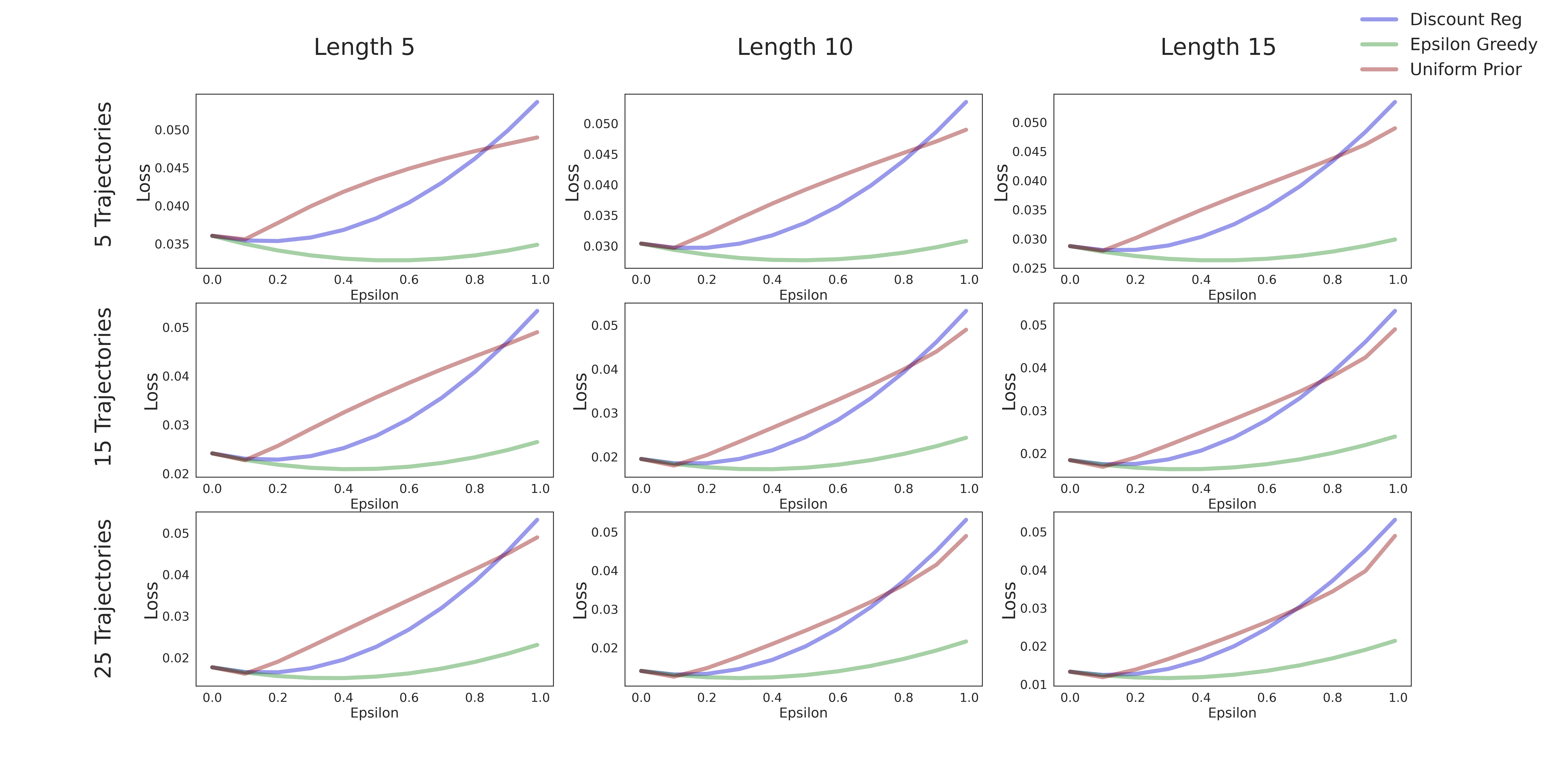}
\vspace{-.5cm}
\caption{\textbf{Two Goals and MSE} varying number and length of trajectories in data set. Random start states, random policy. Discount regularization with absorbing state removed for scale, to show detail on other curves.}
\label{fig:data_amt_TG}
\end{figure}

\subsection{Results by Trajectory Start State}\label{results_start}

\subsubsection{Interconnected Grid}
\begin{figure}[H]
\includegraphics[width=\textwidth]{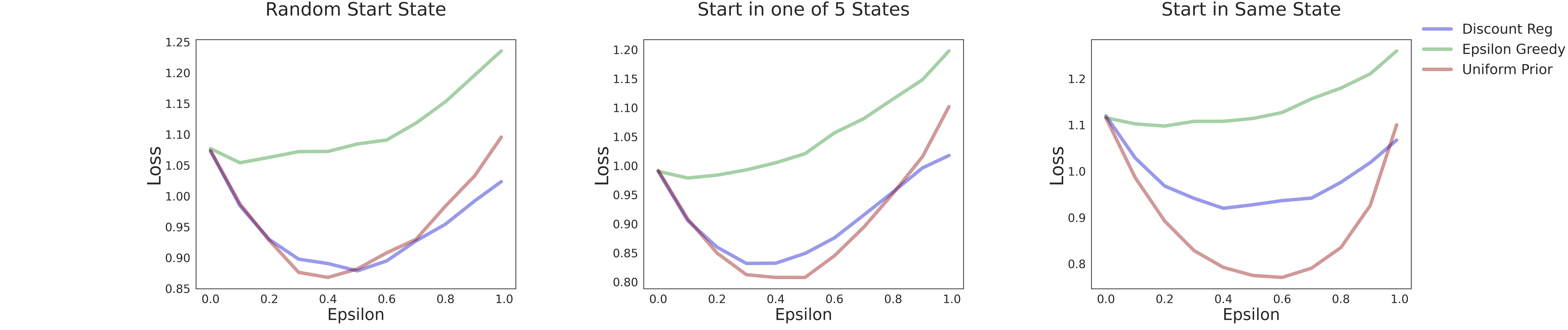}
\caption{\textbf{Interconnected Grid Loss} varying start state. Random policy, 15 trajectories of length 10.}
\end{figure}

\begin{figure}[H]
\includegraphics[width=\textwidth]{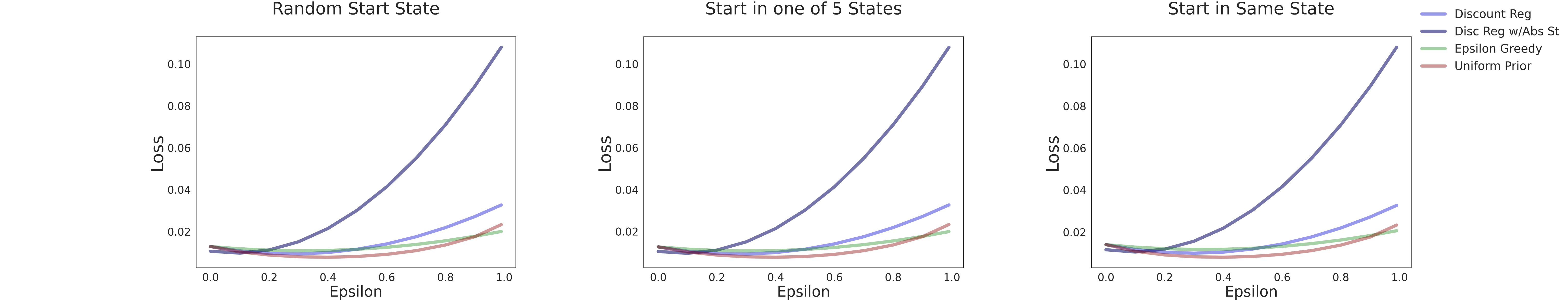}
\caption{\textbf{Interconnected Grid MSE} varying start state. Random policy, 15 trajectories of length 10.}
\end{figure}

\begin{figure}[H]
\includegraphics[width=\textwidth]{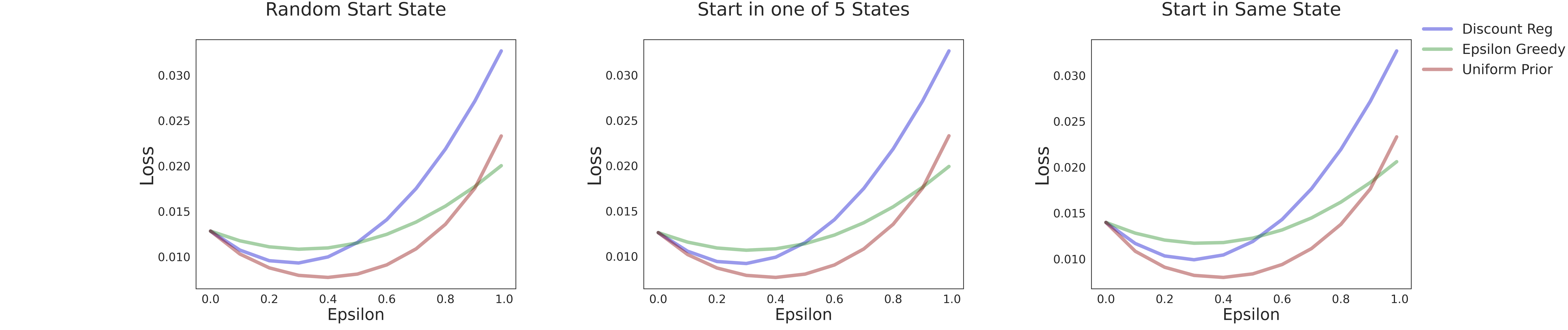}
\caption{\textbf{Interconnected Grid MSE} varying start state. Random policy, 15 trajectories of length 10. Discount regularization with absorbing state removed for scale, to show detail on other curves.}
\end{figure}

\newpage
\subsubsection{Cliff Walk}
\begin{figure}[H]
\includegraphics[width=\textwidth]{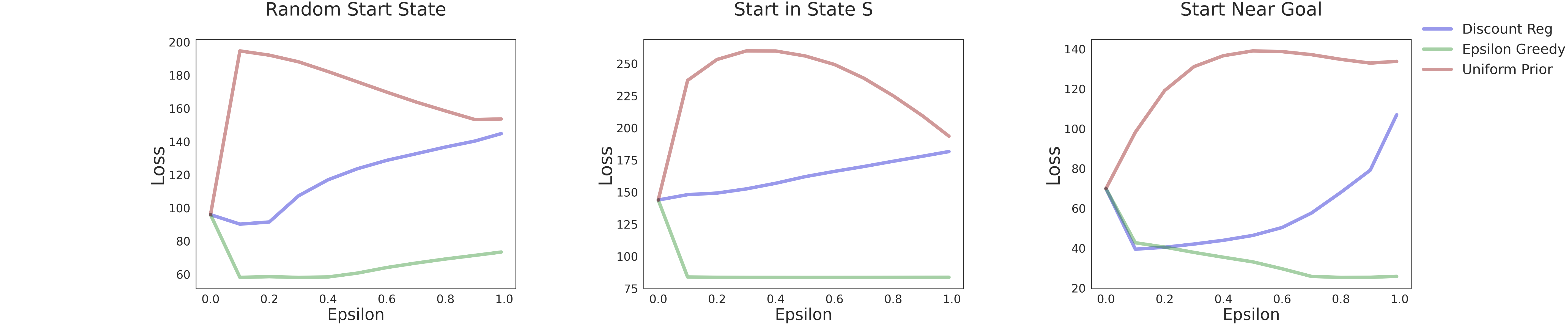}
\caption{\textbf{Cliff Walk Loss} varying start state. Random policy, 25 trajectories of length 20.}
\end{figure}

\begin{figure}[H]
\includegraphics[width=\textwidth]{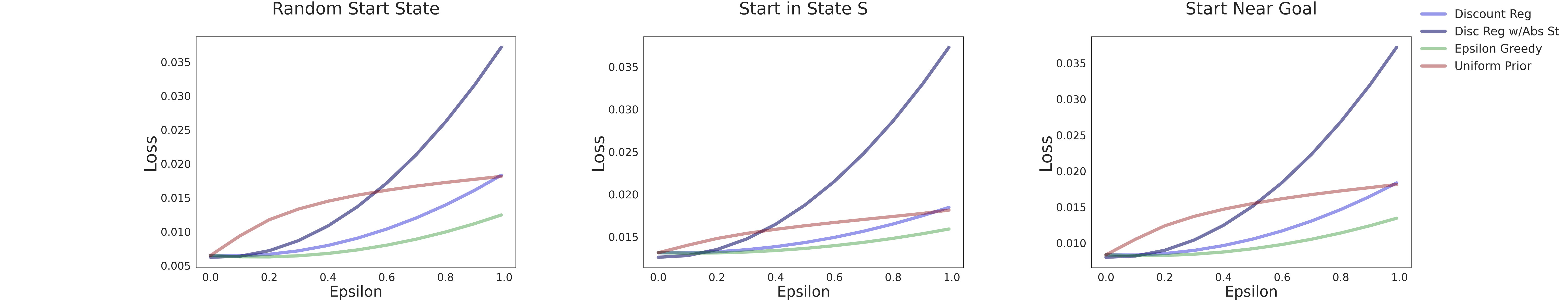}
\caption{\textbf{Cliff Walk MSE} varying start state. Random policy, 25 trajectories of length 20.}
\end{figure}

\begin{figure}[H]
\includegraphics[width=\textwidth]{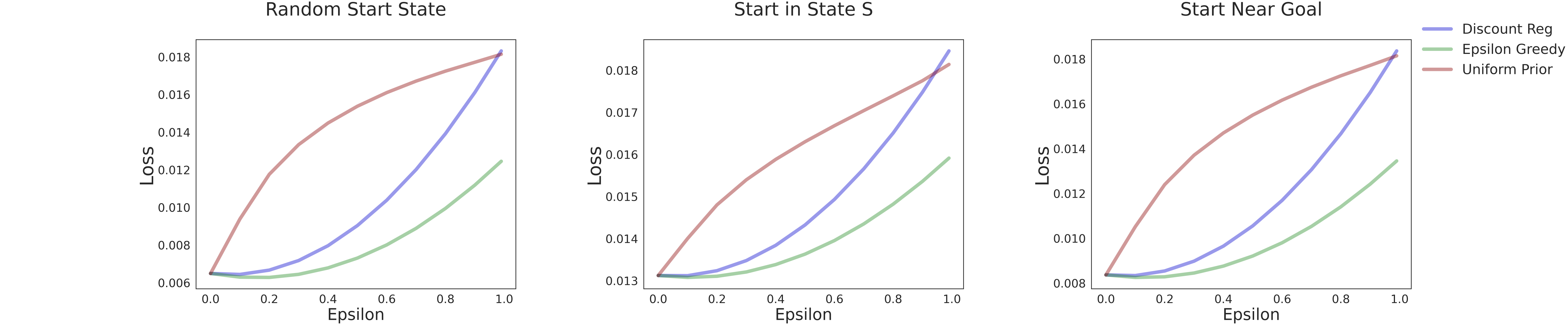}
\caption{\textbf{Cliff Walk MSE} varying start state. Random policy, 25 trajectories of length 20. Discount regularization with absorbing state removed for scale, to show detail on other curves.}
\end{figure}

\subsubsection{Two Goals}
\begin{figure}[H]
\includegraphics[width=\textwidth]{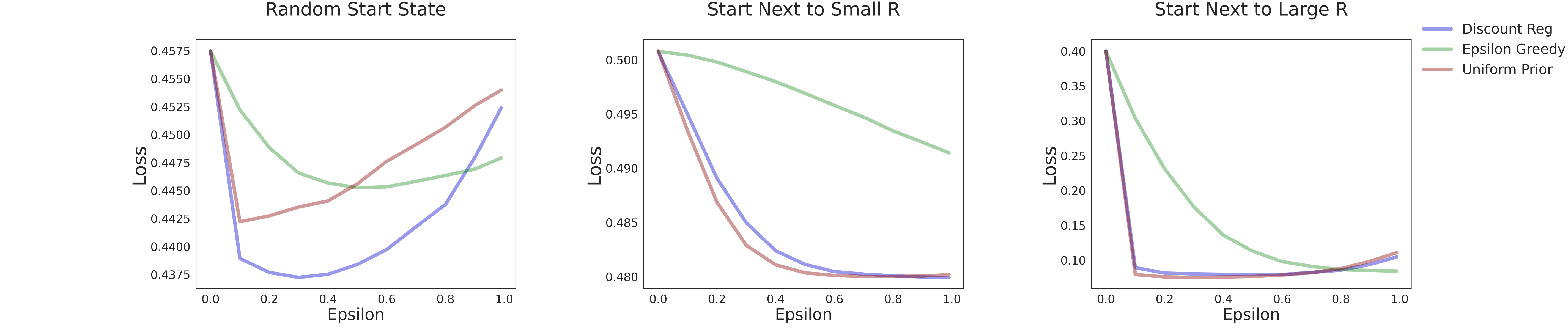}
\caption{\textbf{Two Goals Loss} varying start state. Random policy, 15 trajectories of length 10.}
\end{figure}

\begin{figure}[H]
\includegraphics[width=\textwidth]{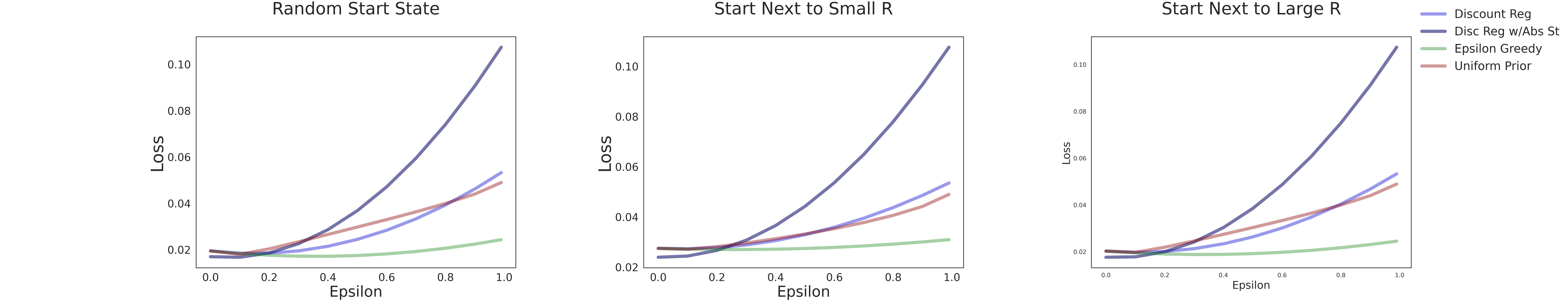}
\caption{\textbf{Two Goals MSE} varying start state. Random policy, 15 trajectories of length 10.}
\end{figure}

\begin{figure}[H]
\includegraphics[width=\textwidth]{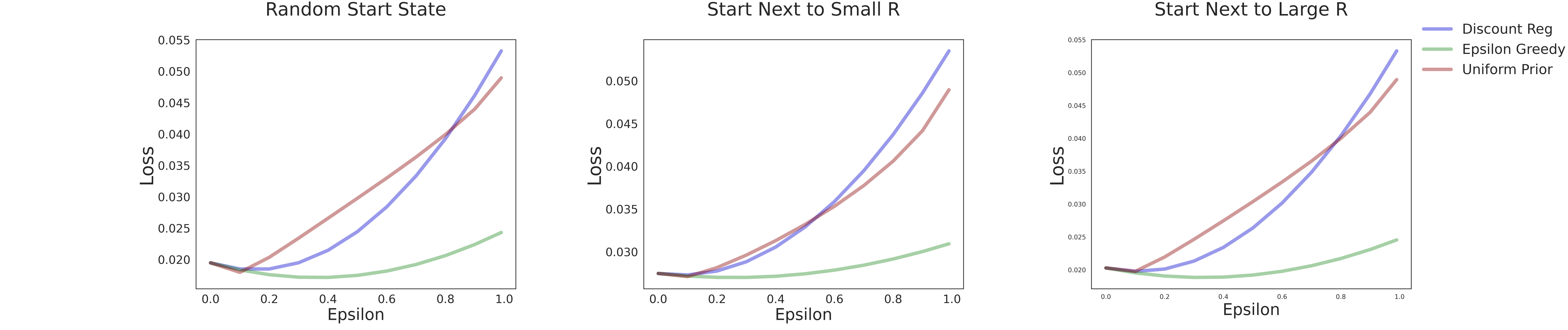}
\caption{\textbf{Two Goals MSE} varying start state. Random policy, 15 trajectories of length 10. Discount regularization with absorbing state removed for scale, to show detail on other curves.}
\end{figure}

\end{document}